%% file: ms.tex
\documentclass[letterpaper]{article} 
\usepackage{aaai25}  
\usepackage{times}  
\usepackage{helvet}  
\usepackage{courier}  
\usepackage[hyphens]{url}  
\usepackage{graphicx} 
\usepackage{natbib}  
\usepackage{caption} 
\usepackage{amssymb,amsfonts,amsmath,amsthm} 
\usepackage{mathtools} 
\usepackage{bm} 
\usepackage{mathrsfs}
\usepackage{thmtools}
\usepackage{thm-restate}

\usepackage{enumerate}
\usepackage{booktabs} 
\usepackage{comment}
\usepackage{multirow}

\usepackage{algorithm}
\usepackage[noend]{algpseudocode}

\newtheorem{theorem}{Theorem}
\newtheorem{lemma}[theorem]{Lemma}

\newtheorem{remark}{Remark}

\title{$p$-Mean Regret for Stochastic Bandits}
\author{
    Anand Krishna\textsuperscript{\rm 1},
    Philips George John\textsuperscript{\rm 2},
    Adarsh Barik\textsuperscript{\rm 3},
    Vincent Y. F. Tan\textsuperscript{\rm 4, 1}
}
\affiliations{
    \textsuperscript{\rm 1} Department of Electrical and Computer Engineering, NUS\\
    \textsuperscript{\rm 2} CNRS-CREATE \& Department of Computer Science, NUS\\
    \textsuperscript{\rm 3} Institute of Data Science, NUS\\
    \textsuperscript{\rm 4} Department of Mathematics, NUS\\
    anandkrishna011095@gmail.com, philips.george.john@u.nus.edu, \{abarik,vtan\}@nus.edu.sg 
}

\newcommand{\p}{$p$}
\newcommand{\E}{\mathbb{E}}
\begin{document}

\maketitle

\begin{abstract}
In this work, we extend the concept of the $p$-mean welfare objective from social choice theory~\cite{moulin2004fair} to study $p$-mean regret in stochastic multi-armed bandit problems. The $p$-mean regret, defined as the difference between the optimal mean among the arms and the $p$-mean of the expected rewards, offers a flexible framework for evaluating bandit algorithms, enabling algorithm designers to balance fairness and efficiency by adjusting the parameter $p$. Our framework encompasses both average cumulative regret and Nash regret as special cases.
We introduce a simple, unified UCB-based algorithm (\textsc{Explore-Then-UCB}) that achieves novel $p$-mean regret bounds. Our algorithm consists of two phases: a carefully calibrated uniform exploration phase to initialize sample means, followed by the UCB1 algorithm of~\citet{auer2002finite}. Under mild assumptions, we prove that our algorithm achieves a $p$-mean regret bound of $\tilde{O}\left(\sqrt{\frac{k}{T^{\frac{1}{2|p|}}}}\right)$ for all $p \leq -1$, where $k$ represents the number of arms and $T$ the time horizon. When $-1<p<0$, we achieve a regret bound of $\tilde{O}\left(\sqrt{\frac{k^{1.5}}{T^{\frac{1}{2}}}}\right)$. For the range $0< p \leq 1$, we achieve a $p$-mean regret scaling as $\tilde{O}\left(\sqrt{\frac{k}{T}}\right)$, which matches the previously established lower bound up to logarithmic factors~\cite{auer1995gambling}. This result stems from the fact that the $p$-mean regret of any algorithm is at least its average cumulative regret for $p \leq 1$.
In the case of Nash regret (the limit as $p$ approaches zero), our unified approach differs from prior work~\cite{barman2023fairness}, which requires a new Nash Confidence Bound algorithm. Notably, we achieve the same regret bound up to constant factors using our more general method.
\end{abstract}

%

\input{introduction}
\input{preliminaries}
\input{technical-overview}

\input{regret-p-gen}
\input{conclusion}

\section{Acknowledgements} 
This research is supported by the National Research Foundation Singapore and DSO National Laboratories under the AI Singapore Programme (AISG Award No: AISG2-RP-2020-018) and a Ministry of Education AcRF Tier 1 grant (Grant No. A-8000189-01-00). PGJ's research is supported by the National Research Foundation, Prime Minister's
Office, Singapore under its Campus for Research Excellence and Technological Enterprise (CREATE) programme.

\bibliography{aaai25}
\onecolumn
\appendix
\input{appendix-supporting}
\input{appendix-nash}

\input{appendix-p-mean-regret}
\input{appendix-experiments}
\end{document}

%% file: introduction.tex
\section{Introduction}\label{sec:intro}

The multi-armed bandit (MAB) problem has long served as a cornerstone in the study of sequential decision-making under uncertainty~\cite{thompson1933likelihood, lai1985asymptotically,bubeck2012regret}. In the stochastic MAB framework, a decision-maker interacts with a set of $k$ arms, each associated with an unknown probability distribution of rewards. Over $T$ rounds, the algorithm selects an arm in each round and receives a reward drawn from that arm's distribution. Ideally, if the distributions were known, the optimal strategy would be to always choose the arm with the highest expected reward. However, because the statistical properties of the arms are unknown, the algorithm's selections may not always be optimal, leading to suboptimal rewards. This performance loss is quantified by the concept of regret, which measures the difference between the optimal mean and the algorithm's performance. Regret thus serves as a performance metric, reflecting the algorithm's efficiency in learning and adapting to the best choices over time. Traditional approaches to MAB problems have primarily focused on maximizing cumulative rewards, often overlooking crucial considerations of fairness and social welfare. As machine learning algorithms increasingly influence resource allocation and decision-making in societally impactful domains—such as healthcare~\cite{villar2015multi}, education~\cite{clement2013multi}, and online advertising~\cite{schwartz2017customer}—there is a growing imperative to incorporate principles of fairness and social welfare into our algorithmic frameworks. 



The MAB literature has extensively studied two main regret formulations: average regret, which measures the difference between the optimal mean and the arithmetic mean of the rewards from the chosen arms over time~\cite{lattimore2020bandit}, and simple regret, which quantifies the expected suboptimality of the final recommended arm, ~\cite{bubeck2009pure}. However, these traditional formulations often fall short in addressing the nuanced fairness considerations that arise in many real-world applications. To appreciate the relevance of a more comprehensive perspective, consider settings where the algorithm's rewards correspond to values distributed across a population of $T$ agents. One such scenario can be found in the context of a large-scale renewable energy initiative. For example, consider a government program testing different types of residential solar panel installations across a diverse region. In each round $t \in \{1, ..., T\}$, a new household is selected to receive a solar panel installation, and the program must choose one of $k$ available solar panel technologies to install for the $t$-th household. The observed reward could be the energy output efficiency of the installed system over a fixed period. While there might be slight variations due to factors like roof orientation or local weather patterns, there is likely one solar panel technology that performs best on average across the entire region. This scenario encapsulates the challenge of balancing between exploiting the currently best-known technology and exploring other options to ensure a potentially superior technology has not been overlooked, all while ensuring fair treatment for each household regardless of when they join the program.

In the context of such applications, the average regret equates the algorithm's performance to the social welfare across all $T$ households. However, this utilitarian approach might not induce a fair outcome, as high social welfare can be achieved even if less efficient solar technologies are inconsiderately installed for an initial set of households. On the other hand, simple regret maps to the egalitarian (Rawlsian) welfare but only after excluding an initial set of households that served as test subjects for various technologies.
These limitations highlight the need for a more flexible and comprehensive framework that can balance fairness and efficiency in MAB settings where the algorithm's rewards correspond to agents' values (e.g., energy output efficiency, household energy savings, or overall carbon footprint reduction). Such a framework should ensure that households joining the program at different times all benefit from increasingly effective technology choices while also allowing for necessary exploration to identify the best solar panel technology for the region.


To address these shortcomings, we study a novel approach that incorporates the concept of \p-mean welfare~\cite{moulin2004fair} from social choice theory into the MAB framework. The \p-mean welfare, also known as the generalized mean or power mean, lends itself as a flexible metric for aggregating individual utilities or rewards. It is defined as: 
$$ W_p = \left( \frac{1}{T} \sum_{t=1}^{T} u_i^p \right)^{\frac{1}{p}},$$ where $p$ is a parameter that determines the sensitivity of the welfare measure to individual utilities $u_i\in \mathbb{R}_{+}$ of each agent $t\in [T]$ . 


This formulation allows for a spectrum of welfare considerations, ranging from the utilitarian approach (emphasizing total utility) to the Rawlsian approach (emphasizing the welfare of the worst-off individual). By adjusting the parameter $p$, one can tailor the welfare measure to reflect different societal preferences. This generalized framework is axiomatically well-supported. Notably, for $p \leq 1$, it adheres to the Pigou-Dalton principle, i.e., redistributing a small portion of reward from an agent with higher utility to one with lower utility tends to enhance overall welfare. The magnitude of this effect intensifies as $p$ decreases, reflecting a stronger emphasis on equity.
Concurrently, the $p$-mean welfare construct allows for a nuanced approach to allocative efficiency. As $p$ increases, the function becomes more sensitive to aggregate reward maximization.
By offering this continuous spectrum between utilitarian and egalitarian objectives, with Nash Social Welfare~\cite{kaneko1979nash} situated as an intermediate case, the $p$-mean welfare framework empowers researchers and decision-makers to fine-tune the balance between efficiency and fairness, tailoring their approach to the specific ethical considerations and practical constraints of diverse application scenarios.

Motivated by the flexibility and comprehensiveness of \p-mean welfare, we utilize this concept in bandit problems by defining \p-mean regret. This metric is given by 
\begin{align}\label{eqn:p-mean}
    R_p = \mu^* - \left( \frac{1}{T} \sum_{t=1}^T \mathbb{E}[\mu_{I_t}]^p \right)^{\frac{1}{p}}
\end{align}  
where \( I_t \) is a random variable representing the arm pulled in round $t$, $ \mu_{I_t}$ is the expected reward of the arm $I_t$, and  $\mu^*$ is the expected reward of the optimal arm. The $p$-mean regret measures the difference between the expected reward of the optimal arm and the \p-mean of the expected rewards of the chosen arms.

The introduction of $p$-mean regret allows for a more nuanced evaluation of bandit algorithms. By varying \( p \), we can capture different aspects of the regret distribution. For instance, when \( p \to 1 \), the $p$-mean regret approaches the traditional regret measure, focusing on the average performance. When \( p \to -\infty \), it emphasizes the performance of the worst-off rounds, analogous to the Rawlsian welfare function. Conversely, when \( p \to +\infty \), it highlights the best-performing rounds, reflecting a more utilitarian perspective without any fairness guarantees. For $p\to 0$, our formulation converges to the Nash regret that was previously studied by~\citet{barman2023fairness}, which is defined as the difference between the optimal mean $\mu^*$ and the geometric mean of the expected rewards. This special case aligns with the Nash social welfare (NSW), an axiomatically-supported welfare objective that satisfies fundamental principles including symmetry, independence of unconcerned agents, scale invariance, and the Pigou-Dalton transfer principle~\cite{moulin2004fair}.

Incorporating $p$-mean regret into the analysis of bandit problems provides a richer understanding of algorithm performance, accounting for diverse criteria of fairness and efficiency. This approach not only aligns with the broader objectives of social choice theory but also offers a robust framework for designing and evaluating bandit algorithms in various practical scenarios.\\

\begin{table}[ht!]
\caption{Summary of results\label{tab:results}}
\begin{tabular}{cll}
\toprule
$p$-range             & Explore period $\Tilde{T}$ & Regret\\
\midrule
$\begin{array}{c} p \in (0, 1] \\ \mbox{(Theorem \ref{thm:positive-regret})} \end{array}$ 
& $16 \sqrt{\frac{T k^p \log T}{\log k}}$ & $\Tilde{O}\left(\sqrt{\frac{k}{T}}\right)$\\
\midrule
$\begin{array}{c} p \to 0 \mbox{ (Nash)}\\ \mbox{(Theorem \ref{thm:nash-regret})}\end{array}$ & $16 \sqrt{\frac{T k\log T}{\log k}}$ & $\Tilde{O}\left(\sqrt{\frac{k}{T}}\right)$\\
\midrule
$\begin{array}{c} p \in [-1, 0)\\ \mbox{(Theorem \ref{thm:negative-regret})}\end{array}$        &  $16 \sqrt{\frac{T\log T}{k^{|p|}}}$ & $\Tilde{O}\left(k^{\frac{3}{4}}{T}^{\frac{-1}{4}}\right)$\\
\midrule
$\begin{array}{c} p \in (-\infty, -1)\\ \mbox{(Theorem \ref{thm:negative-regret})}\end{array}$ & $16 \sqrt{\frac{T\log T}{k^{|p|}}}$ & $\Tilde{O}\left(k^{\frac{1}{2}}{T}^{\frac{-1}{4|p|}}\right)$\\
\bottomrule
\end{tabular}
\end{table}

\section{Our Contributions}

We build upon the work of~\citet{barman2023fairness}, where they develop a new Nash Confidence Bound (NCB) algorithm specifically for Nash regret, i.e., the $p\to 0$ case. In their work, the authors present a stochastic MAB instance where the vanilla UCB1 algorithm incurs a Nash regret of  $1-\frac{1}{T}$. This counterexample illustrates that while UCB is effective for traditional regret metrics, it may not perform well when evaluated using alternative metrics that incorporate fairness or welfare considerations. 

In this work, we develop an algorithm for any $p\leq 1$ by introducing a mild yet significant assumption that every arm has a minimum expected reward of $\Tilde{\Omega}(\sqrt{k/T^{\frac{1}{4}}})$, which allows us to bypass the counter-example where UCB fails. This approach contrasts with the solution proposed in the aforementioned paper where they required a specialized algorithm for handling the case of Nash regret.


Our UCB-based algorithm achieves vanishing $p$-mean regret, i.e., $o(1)$ in the regime where $-\Omega(\log T)<p\leq 1$. As $p$ decreases beyond $-T\log T$, the $p$-mean welfare for any bandit instance becomes a constant factor approximation of egalitarian welfare. Consequently, achieving vanishing $p$-mean regret is not possible with even two arms for $p\leq -T\log T$ (see Proposition 1 in~\citep{barman2020tight}). 

In summary, we explore the concept of \p-mean regret, inspired by \p-mean welfare, and demonstrate the effectiveness of a UCB-based algorithm under this new metric. By overcoming the limitations highlighted in previous work through a minimal assumption, we offer a comprehensive solution that balances fairness and efficiency in the evaluation of bandit algorithms.

We make the following technical contributions.
\begin{itemize}
    \item We study \(p\)-mean regret in stochastic multi-armed bandit problems, providing a flexible framework for evaluating bandit algorithms that balances fairness and efficiency via the parameter \(p\).
    \item We propose \textsc{Explore-Then-UCB}, a unified algorithm designed to achieve novel \(p\)-mean regret bounds. This algorithm consists of two distinct phases: an initial phase involving a calibrated uniform exploration followed by the UCB1 algorithm.
    \item We prove that our algorithm achieves a \(p\)-mean regret bound of \(\tilde{O}\left(\sqrt{\frac{k}{T^{\frac{1}{2|p|}}}}\right)\) for all \(p \leq -1\), where \(k\) is the number of arms and \(T\) is the time horizon and for the range $-1<p<0$, we achieve a regret bound of  \(\tilde{O}\left(\sqrt{\frac{k^{\frac{3}{2}}}{T^{\frac{1}{2}}}}\right)\).
    \item For the special case of Nash regret (\(p \to 0\)), our approach achieves the same regret bound as prior work \cite{barman2023fairness}, but with a simpler algorithm.
    \item For \(0 < p \leq 1\), we show that the \(p\)-mean regret scales as \(\tilde{O}\left(\sqrt{\frac{k}{T}}\right)\), matching the previously proven minimax lower bound that applies to this range via the generalized mean inequality.
\end{itemize}


\section{Related Work}
The incorporation of fairness considerations into MAB problems has garnered significant attention in recent years, driven by the increasing deployment of learning algorithms in domains with far-reaching social implications.


\paragraph{Fairness in Multi-Armed Bandits} Recent works have examined various notions of fairness in MAB contexts. \citet{joseph2016fairness}, \citet{celis2019controlling}, and \citet{patil2021achieving} primarily focused on ensuring fairness for the arms themselves. In a multi-agent setting, \citet{hossain2021fair} and \citet{jones2023efficient} studied scenarios where each arm pull yields potentially different rewards for each agent, aiming to find a fair distribution over arms. While these approaches highlight the importance of fairness in MAB settings, our work diverges by ensuring fairness across time, treating each round as a distinct agent deserving of fair treatment.


\citet{barman2023fairness} introduced the concept of Nash regret and developed the Nash Confidence Bound algorithm to minimize it for stochastic MAB settings. This algorithm provides tight Nash regret guarantees for both known and unknown time horizons (including \(T\)-oblivious settings). Our research extends this work by studying the more general \(p\)-mean regret, which allows for a flexible balance between fairness and efficiency, with Nash regret as a special case when \(p\) approaches 0.



\paragraph{p-Mean Welfare and Fair Division}

The $p$-mean welfare objective has been extensively studied in fair division research, an area at the interface of mathematical economics and computer science. As characterized in social choice theory~\cite{moulin2004fair}, the \(p\)-mean welfare function provides a parameterized framework to balance equity and efficiency~\cite{barman2020tight, garg2021tractable, barman2022universal,eckart2024fairness}. This family of welfare functions is defined by five natural axioms: anonymity, scale invariance, continuity, monotonicity, and symmetry, ensuring that it reflects various principles of fairness in resource allocation. Moreover, the Pigou-Dalton principle---where a transfer from a better-off individual to a worse-off one improves welfare---restricts \(p\) to be less than or equal to 1. Our work leverages this rich theoretical foundation, avoiding the need for ad-hoc fairness constraints. 



\paragraph{Other related work} 
Several recent works have explored fairness concepts in domains adjacent to our research, highlighting the growing interest in incorporating fairness considerations into various learning algorithms. \citet{sawarni2024nash} studied Nash regret for stochastic linear bandits, proving tight upper bounds for sub-Poisson rewards. Their work extends the concept of Nash regret to more complex bandit settings, complementing our generalization to \(p\)-mean regret in the MAB setup.

The work of~\citet{zhang2024no} investigated online multi-agent Nash social welfare (NSW) maximization. Their setting, where the learner's decision affects multiple agents simultaneously, differs from our round-by-round fairness approach but underscores the importance of considering fairness in online decision-making processes.

In a multi-agent reinforcement learning context, \citet{mandal2022socially} adopt an axiomatic approach, demonstrating that Nash Social Welfare (NSW) uniquely satisfies certain fairness axioms and provides regret bounds derived for policies maximizing various fair objectives. These approaches highlight the growing interest in incorporating fairness metrics like NSW into various learning algorithms, from bandits to complex Markov decision process environments.

%% file: preliminaries.tex
\section{Preliminaries}

In the current work, we consider the stochastic MAB problem. Here, we are presented with $k$ unknown probability distributions (arms) each with mean $\mu_i\geq\mu_{\min}$ each supported on the interval $[0, 1]$. Our goal is to design a learning algorithm that adaptively uses sample access (rewards) to these arms over a time horizon $T$ to minimize the $p$-mean regret. This setup corresponds to an agent arriving in every round $t\in [T]$ of the algorithm, which receives a reward $X_t$ based on the arm that the (decision maker) algorithm pulls.

In our approach to \p-mean regret minimization, we employ a dual-phase strategy that combines uniform exploration with the Upper Confidence Bound (UCB1) algorithm. Here, we provide a brief overview of UCB1, introduced by ~\citet{auer2002finite}. It is renowned for its effectiveness in balancing exploration and exploitation in stochastic multi-armed bandit (MAB) problems.

The Upper Confidence Bound (UCB1) algorithm is a fundamental method for solving MAB problems, embodying the principle of `optimism in the face of uncertainty'. UCB1 skillfully balances exploration and exploitation by constructing optimistic estimates of each action's expected payoff. At its core, UCB1 selects the action with the highest upper confidence bound, calculated as \(\hat{\mu}_j + 4\sqrt{\frac{\log T}{n_j}}\), where \(\hat{\mu}_j\) is the empirical mean payoff of arm \(j\) and \(n_j\) is the number of times arm \(j\) has been played. 


By setting the confidence term to 4\(\sqrt{\frac{\log T}{n_j}}\), UCB1 ensures that the probability of the true mean exceeding this bound decreases as \(O(T^{-4})\), rapidly converging to zero as the rounds progress. This approach   addresses the exploration-exploitation dilemma: actions with high empirical means are favored (exploitation), while rarely-played actions maintain high upper bounds due to uncertainty (exploration).


The elegance of UCB1 lies in its simplicity and strong theoretical guarantees, achieving worst-case average regret bound of $O(\sqrt{\frac{k \log T}{T}})$ without prior knowledge of reward distributions. This makes it particularly suitable for our \p-mean regret minimization context.


%% file: technical-overview.tex
\section{The \textsc{Explore-Then-UCB} Algorithm}
Our algorithm begins with a uniform exploration phase, which serves to establish an empirical foundation by collecting initial reward estimates for each arm. This phase is crucial in the \p-mean regret context, where fairness is guaranteed over the entire time horizon. This phase is technically motivated by our use of a stronger property for the confidence-bound (as detailed in Lemmas \ref{lem:ucb_correctness} and \ref{lem:phase_two_good_arms}) in Phase~II, compared to the standard UCB1 analysis. However, striking a balance is crucial, as an overly lengthy exploration phase would negatively impact fairness for a large subset of agents. Consequently, for more negative values of $p$, we employ a  shorter exploration phase to maintain fairness while still benefiting from the initial uniform exploration.



Following this calibrated exploration, the algorithm transitions to UCB1, capitalizing on the gathered data to drive the decision-making process. UCB1 dynamically balances exploration and exploitation based on the observed rewards, adapting effectively to the underlying stochastic environment while ensuring controlled worst-case regret.
This dual-phase strategy achieves novel $p$-mean regret bounds by combining an initial phase of uniform exploration with the proven effectiveness of UCB1. The pseudocode of our algorithm is given in Algorithm \ref{alg:ucb_with_expl}.

\begin{algorithm}
	\caption{The \textsc{Explore-Then-UCB}\label{alg:ucb_with_expl}}
	\begin{algorithmic}
		\State{{\bf Parameters:} Time horizon $T$, number of arms $k$, exploration period $\Tilde{T}$.}
		\For{$t \gets 1,\ldots,\Tilde{T}$}
		\State{Uniformly sample $i_t$ from $[k]$.}
		\State{Pull arm $i_t$ and observe the reward $X_t$.}
		\State{Increment $n_{i_t,t}$ by one and update $\widehat{\mu}_{i_t,t}$}
	\EndFor
	\For{$t \gets \Tilde{T}+1,\ldots,T$}
		\State{Let $\mathsf{UCB}_{i,t-1} \triangleq \widehat{\mu}_{i,t-1} + 4\sqrt{\frac{\log T}{n_{i,t-1}}}$.}
		\State{Select $i_t \in \arg\max_{i \in [k]} \mathsf{UCB}_{i,t-1}$.}
		\State{Pull arm $i_t$ and observe the reward $X_t$.}
		\State{Update $n_{i_t,t}$ and $\widehat{\mu}_{i_t,t}$.}
	\EndFor
	\end{algorithmic}
\end{algorithm}


In this section we show that the \textsc{Explore-then-UCB} algorithm, when configured with suitable exploration periods $\Tilde{T}$, achieves appropriate $o(1)$ regret bounds when $-\Tilde{\Omega}(\log(T))<p\leq 1$. To establish these bounds, we employ the following assumption on the expected rewards.


\begin{restatable}{assumption}{restateassumpsuffposrewa}\label{assump:sufficiently_positive_rewards}
For all arms $i \in [k]$, we have that the expected reward $\mu_i \geq 32\sqrt{\frac{{k\log T\sqrt{\log k}}}{T^{1/4}}}$.
\end{restatable}

In the context of our asymptotic analysis (as $T \to \infty$),  this assumption essentially reduces to a positivity constraint on rewards, i.e., $\mu_i > 0$ for all $i\in [k]$ since we can consider $T$ to be sufficiently large. We also rely on the standard assumption that the exploration period $\Tilde{T}$, and consequently the time horizon \( T \), is sufficiently large compared to \( k \).

\begin{restatable}{assumption}{restateassumpsufflargexplor}\label{assump:sufficiently_large_exploration}
For all bandit instances with $k$ arms where we learn for $T$ time steps, the choice of exploration period $\Tilde{T}$ satisfies $\Tilde{T} \geq 8k\log(Tk) + 16\sqrt{\frac{\sqrt{T}}{\log k}}$.
\end{restatable}

\begin{restatable}{remark}{restateremminttildebound}\label{rem:mu_min_t_tilde_bound}
A consequence of assumptions \ref{assump:sufficiently_positive_rewards} and \ref{assump:sufficiently_large_exploration} is that $\mu_{i} \geq 128\sqrt{\frac{k \log T}{\Tilde{T}}}$ for each $i \in [k]$.
\end{restatable}

Note that, as long as $T$ is sufficiently large (say $T \geq 8k^2\log(k)$) (see Table \ref{tab:results}):
\begin{enumerate}[(i)]
    \item Our choice of $\Tilde{T}$s for $p \in [-1, 1]$ satisfies Assumption \ref{assump:sufficiently_large_exploration} automatically.
    \item Our choice of $\Tilde{T}$ for $p < -1$ satisfies Assumption \ref{assump:sufficiently_large_exploration} as long as $|p| \leq \frac{\log T}{2\log k}$.
\end{enumerate}

\subsection{Regret Analysis}

Analogous to the events used in the analysis of NCB in \cite{barman2023fairness}, we define `good events' $G_1, G_2$ and set $G \triangleq G_1 \cap G_2$. Let $\hat{\mu}_{i,s}$ denote the empirical mean of arm $i$’s rewards \emph{after} seeing $s$ samples from arm $i$, and let $n_{i,t}$ denote the number of times arm $i$ is pulled in steps $\{1,\ldots,t\}$. 

\begin{itemize}
	\item {$G_1$: Every arm $i \in [k]$ is sampled at least $\frac{\tilde{T}}{2k}$ times in Phase I, i.e. $n_{i,\tilde{T}} \geq \frac{\tilde{T}}{2k}$.}
	\item {$G_2$: For all arms $i \in [k]$ (under Assumption \ref{assump:sufficiently_positive_rewards}), and for all sample counts $s$ such that $\frac{\tilde{T}}{2k} \leq s \leq T$, we have $|\mu_i-\widehat{\mu}_{i,s}| \leq  2 \sqrt{\frac{\log T}{s}}$.}
\end{itemize}

Here, we represent all the events in the canonical bandit model~\cite{lattimore2020bandit}. In the following lemma, we show that event $G$ holds with high probability. 
\begin{lemma}\label{lem:good_event_prob}
As long as Assumptions \ref{assump:sufficiently_positive_rewards} and \ref{assump:sufficiently_large_exploration} are satisfied, $P(G) \geq (1- \frac{2}{T})$, where $G = G_1 \cap G_2$.
\end{lemma}

We can now prove the following lemma, which is analogous to Lemma 3 in \cite{barman2023fairness}, but with the UCB bound rather than NCB.

\begin{restatable}[UCB correctness]{lemma}{restatelemucbcorrectness}\label{lem:ucb_correctness}
	Let $\mathrm{UCB}_{i^{*}, t}$ be the upper confidence bound of the optimal arm  $i^{*}$ after round $t$ . Assume that the good event $G$ holds. Then, for all rounds $t > \widetilde{T}$ (i.e., for all rounds in Phase II), we have $\mathrm{UCB}_{i^{*}, t} \geq \mu^{*}$
\end{restatable}

The good event and the above UCB correctness lemma allow us to prove the following lemma, which is a linchpin of our various analyses and is analogous to Lemma 5 of \cite{barman2023fairness}, Lemma 8.2 in \cite{lattimore2020bandit} (but significantly stronger) etc.
\begin{restatable}[Only good arms in phase two]{lemma}{restatelemphasetwogoodarms}\label{lem:phase_two_good_arms}
	Consider a bandit instance that satisfies Assumption \ref{assump:sufficiently_positive_rewards} and assume that the good event $G$ holds. Then, for any arm $i$ that is pulled at least once in Phase II, we have 	
	$$\mu_{i} \geq \mu^{*} - 6 \sqrt{\frac{\log T}{T_{i}-1}},$$ where $T_i$ is the total number of times that arm $i$ is pulled in the algorithm. 
\end{restatable}

We defer the proof of the above supporting lemmas to the appendix. In this paper, we establish upper bounds for $p$-mean regret across the comprehensive range of $p \in (-\infty, 1]$ in the asymptotic ($T \to \infty$) regime (see Theorems \ref{thm:negative-regret}, \ref{thm:positive-regret}, \ref{thm:nash-regret}). 


%% file: regret-p-gen.tex
\section{Regret analysis of $p$-mean regret for $p<0$}

In this section, we establish regret upper bounds for $p$-mean regret when $p$ is negative. We rewrite the $p$-mean regret definition, substituting $q=-p$. That is, we upper bound
$$\mathrm{R}^{q}_T \triangleq \mu^* - \left(\frac{T}{\sum_{t=1}^T \frac{1}{\left(\E_{I_t}[\mu_{I_t}]\right)^q}}\right)^\frac{1}{q}$$
for {$q >0$} and refer to this as $q$-negative-mean-regret. This convention (taking $q > 0$ and considering $q$-negative-mean-regret rather than $(-p)$-mean-regret) is for notational convenience.

The following theorem establishes an upper bound on the $p$-mean regret when $p<0$ employing the \textsc{Explore-Then-UCB} algorithm. Here, we restate the theorem in terms of~$q$. 
\begin{theorem}\label{thm:negative-regret}
Given a bandit instance with $k$ arms, time horizon $T$, and regret parameter $q>0$ where $q = |p|$. By setting the exploration period to {$\Tilde{T} = 16 \sqrt{\frac{T\log T}{k^q}}$}, the $q$-negative-mean regret of the Explore-Then-UCB algorithm satisfies
\begin{align*}
\mathrm{R}^{q}_T &\leq 
\begin{cases}
\Tilde{O}\left(\sqrt\frac{k^{3/2}}{T^{1/2}}\right) &\mbox{when $q \in (0, 1]$} \\[10pt]
\Tilde{O}\left(\sqrt{\frac{k\log(T)}{T^{1/2q}}}\right) &\mbox{when $q > 1$}
\end{cases}
\end{align*}
\end{theorem}
\begin{proof}
Towards analyzing $\mathrm{R}^{q}_T$, we define
\begin{align*}
x \triangleq \frac{T}{\sum_{t=1}^{\Tilde{T}} \frac{1}{\E_{I_t}[\mu_{I_t}]^q}} &\mbox{ and } y \triangleq \frac{T}{\sum_{t=\Tilde{T}+1}^{T} \frac{1}{\E_{I_t}[\mu_{I_t}]^q}},
\end{align*}

so that we have
\begin{equation}\label{eqn:npreg_decomp}
	\mathrm{R}^{q}_T = \mu^* - \left(\frac{1}{\frac{1}{x} + \frac{1}{y}}\right)^{1/q}.
\end{equation}

Hence, to obtain an upper bound for $\mathrm{R}^{q}_T$, we need to upper bound $\frac{1}{x}$ and $\frac{1}{y}$. Let us start by focusing on $\frac{1}{x}$. By uniform exploration in Phase I, we have that 
\begin{align*}
\E_{I_t}[\mu_{I_t}] \geq \frac{\mu^*}{k} \Leftrightarrow \frac{1}{\left(\E_{I_t}[\mu_{I_t}]\right)^q} \leq \left(\frac{k}{\mu^*}\right)^q.
\end{align*}
Hence,
\begin{equation}\label{eqn:npreg_x_inv_bound}
	\frac{1}{x} \leq \frac{\Tilde{T}k^q}{(\mu^*)^q T}.
\end{equation}

We know that by Jensen's inequality ($f(z) = 1/z^q$ is convex on $\mathbb{R}_{>0}$) for $q > 0$ and linearity of expectation
\begin{equation*}
\frac{1}{y} \leq \frac{\sum_{t=\Tilde{T}+1}^T \E_{I_t}\left[\frac{1}{\left(\mu_{I_t}\right)^q}\right]}{T} = \frac{\E_{I_1,\ldots,I_t} \left[\sum_{t=\Tilde{T}+1}^T \frac{1}{\left(\mu_{I_t}\right)^q} \right]}{T}
\end{equation*}

For simplicity, we drop the subscripts in the expectation. By reindexing the arms so that $\{1,2,\ldots,\ell\}$ are the arms pulled at least once in Phase II, and letting $m_i$ be the number of times (the reindexed) arm $i$ is pulled in Phase II, we have 
\begin{equation*}
\frac{\E \left[\sum_{t=\Tilde{T}+1}^{T} \frac{1}{\left(\mu_{I_t}\right)^q} \right]}{T} = \frac{\E \left[\sum_{i=1}^{\ell} \frac{m_i}{\left(\mu_i\right)^q} \right]}{T}.
\end{equation*}

By conditioning on the good event $G$ (see Lemma \ref{lem:good_event_prob}) and noting that $\sum_{i \in [\ell]} m_i = T - \Tilde{T}$ and $\mu_i \geq \mu_{\min}$, we have
\begin{equation*}
\tfrac{\E \left[\sum_{i=1}^\ell \frac{m_i}{(\mu_i)^q} \right]}{T} \leq \tfrac{\E \left[\sum_{i} \frac{m_i}{(\mu_i)^q} \rvert\, G \right]  \mathbb{P}\{G\} + \tfrac{T-\Tilde{T}}{(\mu_{\min})^q}(1-\mathbb{P}\{G\})}{T}
\end{equation*}

We have $1- \mathbb{P}\{G\} \leq \frac{2}{T}$ from Lemma \ref{lem:good_event_prob}. Hence, $\frac{T-\Tilde{T}}{(\mu_{\min})^q}(1-\mathbb{P}\{G\}) \leq \frac{2}{(\mu_{\min})^q}$. Thus, 
\begin{equation}\label{eqn:npreg_yinv_bound}
\frac{1}{y} \leq \frac{\E \left[\sum_{i=1}^\ell \frac{m_i}{(\mu_i)^q} \rvert\, G \right] \cdot \mathbb{P}\{G\}}{T} + \frac{2}{(\mu_{\min})^q T}.
\end{equation}

Consider the first term in RHS. Conditioned on event $G$, by Lemma \ref{lem:phase_two_good_arms}, we have $\mu_i\geq \mu^*-\beta_i$ for all arms $i\in [\ell]$ pulled in Phase II --- where $\beta_i \triangleq 6\sqrt{\frac{\log T}{T_i-1}}$. Note that, conditioned on the good event $G$ (specifically $G_1$), we have $$\beta_i = 6\sqrt{\frac{\log T}{T_i-1}} \leq 6\sqrt{\frac{2k\log T}{\Tilde{T}}} \triangleq \beta$$
for any arm $i \in [\ell]$. Hence, we have 
\begin{align*}
	\frac{\E [\sum_{i=1}^\ell \frac{m_i}{(\mu_i)^q} \rvert\, G ] \mathbb{P}\{G\}}{T} &\leq \frac{\E [\sum_{i=1}^\ell \frac{m_i}{({\mu^*} - \beta)^q}]}{T},
\end{align*}
using $\mathbb{P}(G) \leq 1$ and $\mu_i \geq \mu^\ast - \beta$. Note that Assumption \ref{assump:sufficiently_positive_rewards} implies that $\mu_i - \beta > 0$ for all $i \in [k]$, and in particular that $\mu^* - \beta > 0$. Also, from the assumptions we can see that $\frac{\beta}{\mu^\ast} \leq \frac{1}{2}$. Note that we have $(\mu^*-\beta)^q \geq ({\mu^*})^q - q\beta({\mu^*})^{q-1}$ for $q>0$. Combining this with  (\ref{eqn:npreg_x_inv_bound}) and (\ref{eqn:npreg_yinv_bound}) we get
\begin{align*}
    \frac{1}{\frac{1}{x}+\frac{1}{y}} &\geq \frac{1}{\frac{\Tilde{T}k^q}{(\mu^*)^q T} + \frac{\E [\sum_{i=1}^\ell m_i]}{T\left((\mu^*)^q - q\beta(\mu^*)^{q-1}\right)} + \frac{c}{T}} \tag{taking $\frac{2}{(\mu_{\min})^q} \triangleq c$}\\
	&\geq \frac{(\mu^*-\beta) T}{k^p\Tilde{T} + (T-\Tilde{T}) + c} \triangleq (\ast_1)\tag{cancelling $\mu^\ast$ and using $c \geq c(\mu^*)^q$}
	\end{align*}
	Continuing, we can divide the numerator and denominator by $T$ and rearrange the terms to get
	\begin{align*}
	(\ast_1) &=  \frac{((\mu^*)^q-q\beta(\mu^*)^{q-1})}{(k^q\frac{\Tilde{T}}{T} + \frac{(T-\Tilde{T})}{T} + \frac{c}{T})}\\
	&=  \frac{((\mu^*)^q-q\beta(\mu^*)^{q-1})}{\left(1 + \frac{(k^q-1)\Tilde{T}+c}{T}\right)}\\
	&= \frac{((\mu^*)^q-q\beta(\mu^*)^{q-1})\left(1-\frac{(k^q-1)\Tilde{T}+c}{T}\right)}{\left(1 - \left(\frac{(k^q-1)\Tilde{T}+c}{T}\right)^2\right)} \triangleq (\ast_2)
	\end{align*}

 In the last step, we multiply both numerator and denominator by $\left(1 - \frac{(k^q-1)\Tilde{T}+c}{T}\right) \in (0, 1)$. Thus, $1 - \left(\frac{(k^q-1)\Tilde{T}+c}{T}\right)^2 \in (0, 1)$ and we get
 \begin{align*}
	(\ast_2) &\geq (\mu^*)^q\left(1-\frac{q\beta}{\mu^*}\right)\left(1-\frac{(k^q-1)\Tilde{T}+c}{T}\right) \triangleq (\ast_3)\\
\end{align*}
	


\subsubsection*{When $q \geq 1$}
We can expand $(\ast_3)$ to get
\begin{align*}
	\frac{1}{\tfrac{1}{x} + \tfrac{1}{y}} &\geq
   (\mu^*)^q\left(1-\tfrac{(k^q-1)\Tilde{T}+c}{T}-\tfrac{q\beta}{\mu^*}+\tfrac{q\beta((k^q-1)\Tilde{T}+c)}{\mu^*T}\right)\\
   &\geq (\mu^*)^q\left(1-\tfrac{(k^q-1)\Tilde{T}+c}{T}-\tfrac{q\beta}{\mu^*}\right)\\
   &= (\mu^*)^q\left(1-\tfrac{(k^q-1)\Tilde{T}+c}{T}-\tfrac{6q\sqrt{2k\log T}}{\sqrt{\Tilde{T}}\mu^*}\right)
\end{align*}

Thus, we have (for $q > 1$), that $\left(\tfrac{1}{\frac{1}{x} + \tfrac{1}{y}}\right)^{\tfrac{1}{q}}$ is
\begin{align*}
&\geq \left((\mu^*)^q -(\mu^*)^q\left(\tfrac{(k^q-1)\Tilde{T}+c}{T}+\tfrac{6q\sqrt{2k\log T}}{\sqrt{\Tilde{T}}\mu^*}\right)\right)^{\frac{1}{q}}\\
&\geq \mu^* - \mu^*\left(\tfrac{(k^q-1)\Tilde{T}+c}{T}+\tfrac{6q\sqrt{2k\log T}}{\sqrt{\Tilde{T}}\mu^*}\right)^{1/q}
\end{align*}
The last inequality uses the fact that $(a - b)^{1/q} \geq a^{1/q} - b^{1/q}$ for all $a \geq b \geq 0$ (from binomial expansion) as long as $q \geq 1$. Then, we have from (\ref{eqn:npreg_decomp}) that
\begin{align*}
R_T^q  &\leq \mu^*\left(\frac{(k^q-1)\Tilde{T}+c}{T}+\frac{6q\sqrt{2k\log T}}{\sqrt{\Tilde{T}}\mu^*}\right)^{1/q}.
\end{align*}

Since $\Tilde{T} = 16\sqrt{\frac{T\log T}{k^q}}$, we can substitute the value to get
\begin{align*}
	R_T^q  &\leq \mu^*\left(\tfrac{16(k^q-1)\sqrt{T\log T}}{\sqrt{k^q}T}+\tfrac{c}{T}+\tfrac{6q\sqrt{2k\log T} k^{q/4}}{4\left(T \log T\right)^{1/4}\mu^*}\right)^{\frac{1}{q}}\\
    &\leq \Tilde{O}\left(\frac{k^{\frac{1}{2}}}{T^{\frac{1}{4q}}}\right)\tag{since $\frac{1}{4}+\frac{1}{2q} \leq \frac{1}{2}$ for $q \geq 1$}.
\end{align*}

\subsubsection*{When $0 < q < 1$}

From $(\ast_3)$, we get
\begin{align*}
\left(\tfrac{1}{\frac{1}{x} + \frac{1}{y}}\right)^{\frac{1}{q}} \geq (\mu^*)\left(1-\tfrac{q\beta}{\mu^*}\right)^{\frac{1}{q}}\left(1-\tfrac{(k^q-1)\Tilde{T}+c}{T}\right)^{\frac{1}{q}} \triangleq (\ast_4)
\end{align*}
Since $0 < q \leq 1$, $\frac{q\beta}{\mu^*} \leq \frac{1}{2}$ and $\frac{(k^q-1)\Tilde{T}+c}{T} \leq \frac{1}{2}$, we can use Weierstrass inequality~\cite{kozma2021useful} to get
\begin{align*}
(\ast_4) &\geq  \mu^* \left(1-\frac{\beta}{\mu^*}-\frac{(k^q-1)\Tilde{T}+c}{qT}\right)\\
&\geq \mu^* - \mu^*\left(\frac{\beta}{\mu^*}+\frac{(k^q-1)\Tilde{T}}{qT}+\frac{c}{qT}\right)
\end{align*}

Since $\Tilde{T} = 16\sqrt{\frac{T\log T}{k^q}}$, we can substitute the value to get

\begin{align*}
R^{q}_T &\leq \mu^*\left(\frac{\beta}{\mu^*}+\frac{(k^q-1)\Tilde{T}+c}{qT}\right)\\
&\leq \frac{6\sqrt{2k\log T}}{\sqrt{\Tilde{T}}} + \frac{(k^q-1)\Tilde{T}}{qT} +\frac{c}{qT}\\
&\leq \frac{6\sqrt{2k\log T} k^{q/4}}{4\left(T\log T\right)^{1/4}}+ \frac{(k^q-1)16\sqrt{\log T}}{k^{\frac{q}{2}}q\sqrt{T}} +\frac{c}{qT}\\
&\leq \Tilde{O}\left(\frac{k^{\frac{3}{4}}}{T^{\frac{1}{4}}}\right) \tag{noting $\max\{\frac{1}{2}+\frac{q}{4},\frac{q}{2}\} \leq \frac{3}{4}$ and for $q \geq 4/(kT)^{\frac{3}{4}}$}
\end{align*}

\begin{remark}
Our bound for $p$-mean regret $R^{p}_T$ when $p < -1$ is $\lessapprox \frac{\sqrt{k\log T}}{T^{1/4|p|}}$. For any $\kappa > 0$, this bound will be at least $\kappa$ when $$|p| \!\geq\! \frac{\log T}{4 \left(\log(\sqrt{k \log T}) \!+\! \log\left(\frac{1}{\kappa}\right)\right)} \!\gtrapprox \!\frac{\log T}{\log \log T \!+\! \log k \!+\! \log\left(\frac{1}{\kappa}\right)}.$$ This provides the transition point when the algorithm stops achieving vanishing regret.
\end{remark}

\end{proof}

\section{Regret analysis for $p$-mean regret for $0 < p \leq 1$}

In this section, we establish regret upper bounds for $p$-mean regret when $p$ is between $0$ and $1$. That is, we upper bound
$$\mathrm{R}^{p}_T = \mu^* - \left(\frac{\sum_{t=1}^T \left(\E_{I_t}[\mu_{I_t}]\right)^p}{T}\right)^\frac{1}{p}$$
for {$p \in (0, 1]$}. The following theorem establishes an ${\Tilde{O}\left(\sqrt{\frac{k}{T}}\right)}$ upper bound on $p$-mean regret when using the \textsc{Explore-Then-UCB} algorithm. We sketch the proof here and give all the details in the appendix.

\begin{theorem}\label{thm:positive-regret}
Given a bandit instance with $K$ arms, time horizon $T$, and regret parameter $p \in (0, 1]$, choosing the exploration period {$\Tilde{T} = 16 \sqrt{\frac{T k^p \log T}{\log k}}$}, the $p$-mean regret of the \textsc{Explore-Then-UCB} algorithm satisfies
\begin{align*}
\mathrm{R}^{p}_T &\leq \Tilde{O}\left(\sqrt{\frac{k}{T}}\right).
\end{align*}
\end{theorem}

\paragraph{Proof Sketch.} Towards analyzing $\mathrm{R}^{p}_T$, we define 
\begin{align*}
x \triangleq \frac{\sum_{t=1}^{\Tilde{T}}\E_{I_t}[\mu_{I_t}]^p}{T} \text{ and } y \triangleq \frac{\sum_{t=\Tilde{T}+1}^{T}\E_{I_t}[\mu_{I_t}]^p}{T},
\end{align*}
so that we have
\begin{equation}\label{eqn:preg_decomp}
	\mathrm{R}^{p}_T = \mu^* - (x+y)^{1/p}.
\end{equation}

To obtain an upper bound for $\mathrm{R}^{p}_T$, we need to lower bound $x$ and $y$. We show that $x \geq \frac{{\mu^*}^p \Tilde{T}}{k^p T}$ since we use uniform exploration in Phase I. By reindexing the arms so that $\{1,2,\ldots,\ell\}$ are the arms pulled at least once in Phase II, and letting $m_i$ be the number of times (the reindexed) arm $i$ is pulled in Phase II, we have $y \geq \frac{\E [\sum_{i=1}^{\ell} m_i \mu_{i}^p \rvert\, G ] \mathbb{P}\{G\}}{T}$. Using these facts and simplifying further, as in the case of $p<0$ we get 
\begin{align*}
(x+y)^{\frac{1}{p}} \geq {\mu^*} - \frac{(k^p-1)\Tilde{T'}}{p\sqrt{k}T} -\frac{6\sqrt{k\log T}}{\sqrt{T}} - \frac{4}{pT}
\end{align*}
where $\Tilde{T'}= \frac{\Tilde{T}}{k^{p-0.5}}$.
Then, substituting for $\Tilde{T}$,
\begin{align*}
\textrm{R}^p_T &= \mu^* - (x+y)^{1/p}\\
&\leq \frac{(k^p-1)\Tilde{T'}}{p\sqrt{k}T} + \frac{6\sqrt{k\log T}}{\sqrt{T}} + \frac{4}{pT}\\
&= \frac{(k^p-1)\Tilde{T}}{pk^{p}T} + \frac{6\sqrt{k\log T}}{\sqrt{T}} + \frac{4}{pT}\\
&= \frac{16(k^p-1)\sqrt{T k^p \log(T)}}{pk^{p}T\sqrt{\log k}} + \frac{6\sqrt{k\log T}}{\sqrt{T}} + \frac{4}{pT}\\
&\leq \Tilde{O}\left(\sqrt{\frac{k}{T}}\right)\tag{for $p\geq \frac{4}{\sqrt{kT}}$}
\end{align*}



We also provide a bound for Nash regret (proof in appendix), as studied in \cite{barman2023fairness}, achieving essentially the same regret bound as theirs but under our assumptions and using the \textsc{Explore-Then-UCB} algorithm instead of their NCB algorithm, which employs a different confidence bound following the exploration phase.

\begin{theorem}\label{thm:nash-regret}
Given a bandit instance with $k$ arms and for time horizon $T$, choosing the exploration period {$\Tilde{T} = 16 \sqrt{\frac{T k \log T}{\log k}}$}, the Nash regret of the \textsc{Explore-Then-UCB} algorithm satisfies
\begin{align*}
\mathrm{NR}_T &\leq \Tilde{O}\left({\sqrt\frac{k}{T}}\right).
\end{align*}
\end{theorem}

%% file: conclusion.tex
\section{Conclusion}

Building on the $p$-mean welfare concept from social choice theory, this work examined $p$-mean regret—a flexible metric that allows the decision maker to effectively balance fairness and efficiency considerations. We proposed a unified \textsc{Explore-Then-UCB} algorithm that achieves $p$-mean regret bounds across a wide range of $p$ values. Specifically, our analysis demonstrates that for \( p \leq -1 \), the \p-mean regret of our algorithm scales as \( \tilde{O}(k^{\frac{1}{2}} T^{-\frac{1}{4|p|}}) \), while for \( 0 \leq p \leq 1 \), the regret scales as \( \tilde{O}(k^{\frac{1}{2}} T^{-\frac{1}{2}}) \). This unified approach simplifies the design and analysis of bandit algorithms, particularly when compared to prior work that required specialized techniques for the Nash regret case \( (p \to 0) \).

Several promising directions for future work emerge from our findings. Extending the analysis of $p$-mean regret to other bandit settings, such as linear and contextual bandits, could provide deeper insights into the interplay between fairness, efficiency, and bandit structure. Developing tight meta-algorithms that replace the UCB1 subroutine with other average-regret minimizing algorithms could be another interesting direction for future work. Extending the work for anytime guarantees and unknown time horizons would also improve the practical applicability of our work. By continuing to explore the connections between fairness, efficiency, and bandit learning, the study of \p-mean regret contributes to the development of more socially responsible and widely applicable decision-making algorithms.

%% file: appendix-supporting.tex
\section{Proofs of the Supporting Lemmas}\label{app:supporting}

We restate the assumptions for the analysis here. 

\restateassumpsuffposrewa*

In the context of our asymptotic analysis (as $T \to \infty$),  this assumption essentially reduces to a positivity constraint on rewards, i.e., $\mu_i > 0$ for all $i\in [k]$ since we can consider $T$ to be sufficiently large. We also rely on the standard assumption that the exploration period $\Tilde{T}$, and consequently the time horizon \( T \), is sufficiently large compared to \( k \).

\restateassumpsufflargexplor*

\restateremminttildebound*

We define `good events' $G_1, G_2$ and set $G \triangleq G_1 \cap G_2$. Let $\hat{\mu}_{i,s}$ denote the empirical mean of arm $i$’s rewards \emph{after} seeing $s$ samples from arm $i$, and let $n_{i,t}$ denote the number of times arm $i$ is pulled in steps $\{1,\ldots,t\}$. 

\begin{itemize}
	\item {$G_1$: Every arm $i \in [k]$ is sampled at least $\frac{\tilde{T}}{2k}$ times in Phase I, i.e. $n_{i,\tilde{T}} \geq \frac{\tilde{T}}{2k}$.}
	\item {$G_2$: For all arms $i \in [k]$ (under Assumption \ref{assump:sufficiently_positive_rewards}), and for all sample counts $s$ such that $\frac{\tilde{T}}{2k} \leq s \leq T$, we have $|\mu_i-\widehat{\mu}_{i,s}| \leq  2 \sqrt{\frac{\log T}{s}}$.}
\end{itemize}

Here, we represent all the events in the canonical bandit model~\cite{lattimore2020bandit}. In the following lemma, we show that event $G$ holds with high probability. 

We restate and prove Lemma \ref{lem:ucb_correctness} here, which establishes that the good event $G = G_1 \cap G_2$ actually holds with high probability.

\restatelemucbcorrectness*
\begin{proof}
	\textbf{(Proving $\Pr(\neg G_1) \leq 1/T$)} Let $Z_{i,r} \triangleq \mathbb{I}[\mbox{arm $i$ is selected at round $r$}]$, for all $i \in [k]$ and $r \in [\widetilde{T}]$. Then the number of times arm $i$ is sampled in Phase I is given by $Z_i \triangleq \sum_{r \in [\tilde{T}]} Z_{i.r}$. This is a sum of i.i.d Bernoulli random variables with $\E[Z_i] = \frac{\tilde{T}}{k}$. Then, by the multiplicative Chernoff bound, $\Pr(Z_i \leq (1 - \frac12)\E[Z_i]) \leq \exp(-\frac{(\frac12)^2 \tilde{T}}{2k}) \leq \exp(-\frac{\tilde{T}}{8k}) \leq \frac{1}{Tk}$ since $\tilde{T} \geq 8k\log(Tk)$ by Assumption \ref{assump:sufficiently_large_exploration}. Applying the union bound over $k$ events $\{Z_i \leq \frac12 \E[Z_i]\}_{i \in [k]}$, we get $\Pr(\neg G_1) \leq 1/T$.
	
	\textbf{(Proving $\Pr(\neg G_2) \leq 1/T$)} For each arm $i \in [K]$ with mean reward $\mu_i \in [0, 1]$, define $\rho_{i,j} \in [0, 1/s]$ to be the reward obtained on pulling arm $i$ the $j$-th time divided by $s$, so that $\widehat{\mu}_{i,s} = \sum_{j=1}^{s} \rho_{i,j}$. Since this is a sum of iid bounded random variables, we can apply the additive Hoeffding bound to get
	
	\begin{align*}
		\mathbb{P}\left\{\left|\mu_{i}-\widehat{\mu}_{i, s}\right| \geq C_2\sqrt{\frac{\log T}{s}}\right\} \leq 2 \exp\left(-\frac{2C_2^2\left(\frac{\log T}{s}\right)}{s\left(\frac{1}{s}-0\right)^2}\right)  \leq \frac{2}{T^{2C_2^2}}
	\end{align*}
	Set $C_2 = 2$. Using union bound we get $\mathbb{P}\left\{G_{2}^{c}\right\} \leq \frac{2}{T^{4}} k T \leq \frac{1}{T}$ (since $T \geq 2k$ by Assumption \ref{assump:sufficiently_large_exploration}).
	
	
	Thus, by De Morgan's law and the union bound, $\Pr(\neg G) \leq \Pr(\neg G_1) + \Pr(\neg G_2) = 2/T$, which gives us the desired probability for $G$.
\end{proof}

We now restate and prove Lemma \ref{lem:phase_two_good_arms}, that establishes the correctness of the UCB bound during Phase II under the good event $G$.

\restatelemphasetwogoodarms*
\begin{proof}
	\begin{align*} 
		\mathrm{UCB}_{i^{*}, t} & =\widehat{\mu}_t^{*} + 4 \sqrt{\frac{\log T}{n_{i^{*},t}}} \\ 
		& \geq \mu^\ast - 2 \sqrt{\frac{\log T}{n_{i^*,t}}} + 4 \sqrt{\frac{\log T}{n_{i^{*},t}}}\,\,\,(\mbox{by $G_2$})\\
		&\geq \mu^\ast
	\end{align*}
	Note that $\widehat{\mu}^\ast_{t} = \widehat{\mu}_{i^\ast,t}$ is the same as the empirical mean after seeing $n_{i^\ast,t}$ samples from arm $i^\ast$. We can apply $G_2$ since, by $G_1$, $n_{i^\ast,t} \geq \Tilde{T}/2k$ for all $t > \Tilde{T}$.
\end{proof}


The following lemma provides a lower bound on the expected rewards of arms selected during Phase II, ensuring that they are close to the optimal reward under the good event $G$.

\paragraph*{Lemma \ref{lem:phase_two_good_arms} (restated):} Consider a bandit instance which satisfies Assumption \ref{assump:sufficiently_positive_rewards} and assume that the good event $G$ holds. Then, for any arm $i$ that is pulled at least once in Phase II, we have 
	$$\mu_{i} \geq \mu^{*} - 6 \sqrt{\frac{\log T}{T_{i}-1}},$$ where $T_i$ is the total number of times that arm $i$ is pulled in the algorithm. 
\begin{proof}
	For time step $t$ in Phase II where arm $i$ is pulled for the last time (i.e. a total of $T_i$ times), $\mathrm{UCB}_{i,t-1} \geq \mathrm{UCB}_{i^*,t-1} \geq \mu^*$ (the last inequality from Lemma \ref{lem:ucb_correctness}).
	
	That is, before this particular round $t$, we have $\widehat{\mu}_{i,t-1} + 4 \sqrt{\frac{\log T}{T_{i}-1}} \geq \mu^{*}.$
	
	Then
	\begin{align*}
		\mu^{*} &\leq \widehat{\mu}_{i,t-1} + 4 \sqrt{\frac{\log T}{T_{i}-1}}\\
		&\leq \mu_i + 2\sqrt{\frac{\log T}{T_{i}-1}} + 4 \sqrt{\frac{\log T}{T_{i}-1}}\,\,\,(\mbox{by $G_2$})\\
		&\leq \mu_i + 6\sqrt{\frac{\log T}{T_{i}-1}}
	\end{align*}
	
	Note that $\widehat{\mu}_{i,t-1}$ is the same as the empirical mean after seeing $T_i-1$ samples from arm $i$.
\end{proof}

%% file: appendix-nash.tex
\section{Regret Analysis of Explore-then-UCB: Nash Regret}\label{app:nash}

\begin{theorem}
Given a bandit instance with $k$ arms and time horizon $T$, choosing the exploration period $\Tilde{T} = 16\sqrt\frac{T k \log T}{\log k}$, the Nash regret of the \textsc{Explore-Then-UCB} algorithm satisfies
$$\mathrm{NR}_T = \Tilde{O}\left(\sqrt{\frac{k}{T}}\right).$$ 
\end{theorem}

\begin{proof}
Note that this analysis is similar to that of \cite{barman2023fairness} for the NCB algorithm, but with appropriate modifications for using the standard UCB bound in the second phase. We include it here in the appendix for completeness.
 
From the definition of Nash Regret, we have
\begin{equation}\label{eqn:nash_regret_def_app}
\mathrm{NR}_T = \mu^* - \left(\prod_{t=1}^{T} \mathbb{E}\left[\mu_{I_{t}}\right]\right)^{\frac{1}{T}}.
\end{equation}

We can decompose the Nash Social Welfare (geometric mean) as
\begin{equation}\label{eqn:nash_sw_decomp}
\left(\prod_{t=1}^{T} \mathbb{E}\left[\mu_{I_{t}}\right]\right)^{\frac{1}{T}} = \left(\prod_{t=1}^{\widetilde{T}} \mathbb{E}\left[\mu_{I_{t}}\right]\right)^{\frac{1}{T}}\left(\prod_{t=\widetilde{T}+1}^{T} \mathbb{E}\left[\mu_{I_{t}}\right]\right)^{\frac{1}{T}}.
\end{equation}
We will lower bound the two RHS factors in the decomposition separately.

Since we use the same $\Tilde{T}$, the Phase I (uniform exploration) of the \textsc{Explore-Then-UCB} is identical to that of the Nash Confidence Bound (NCB) algorithm \cite{barman2023fairness}.
Thus, the analysis for Phase I holds as is, i.e., 
\begin{equation}\label{eqn:nash_sw_phase_one}
\left(\prod_{t=1}^{\widetilde{T}} \mathbb{E}\left[\mu_{I_{t}}\right]\right)^{\frac{1}{T}}\geq \left(\mu^{*}\right)^{\frac{\tilde{T}}{T}}\left(1-\frac{16 \sqrt{k \log k \log T}}{\sqrt{T}}\right).
\end{equation}

In Phase II,
\begin{align*} 
\left(\prod_{t=\widetilde{T}+1}^{T} \mathbb{E}\left[\mu_{I_{t}}\right]\right)^{\frac{1}{T}} \geq \mathbb{E}\left[\left(\prod_{t=\widetilde{T}+1}^{T} \mu_{I_{t}}\right)^{\frac{1}{T}}\right] \geq \mathbb{E}\left[\left.\left(\prod_{t=\widetilde{T}+1}^{T} \mu_{I_{t}}\right)^{\frac{1}{T}} \right\rvert\, G\right] \mathbb{P}\{G\},
\end{align*}
where we use Jensen's inequality and the law of conditional expectation.

Let the arms pulled at least once in Phase II be denoted $\{1,\ldots,\ell\}$ (by reindexing) and let $m_i \geq 1$ denote the number of times arm $i \in [\ell]$ is pulled in Phase II. Note that $\sum_{i=1}^\ell m_i = T - \widetilde{T}$. So we can write
\begin{align*}
\mathbb{E}\left[\left.\left(\prod_{t=\widetilde{T}}^{T} \mu_{I_{t}}\right)^{\frac{1}{T}} \right\rvert\, G\right]=\mathbb{E}\left[\left.\left(\prod_{i=1}^{\ell} \mu_{i}^{\frac{m_{i}}{T}}\right) \right\rvert\, G\right]
\end{align*}

Hence using Lemma \ref{lem:phase_two_good_arms},
\begin{align*}
\mathbb{E}\left[\left.\left(\prod_{t=\widetilde{T}}^{T} \mu_{I_{t}}\right)^{\frac{1}{T}} \right\rvert\, G\right] &= \mathbb{E}\left[\left.\left(\prod_{i=1}^{\ell} \mu_{i}^{\frac{m_{i}}{T}}\right) \right\rvert\, G\right] \geq\mathbb{E}\left[\left.\prod_{i=1}^{\ell}\left(\mu^{*}-6 \sqrt{\frac{ \log T}{T_{i}-1}}\right)^{\frac{m_{i}}{T}} \right\rvert\, G\right]\\
&\geq (\mu^{*})^{1 - \frac{\widetilde{T}}{T}}\mathbb{E}\left[\left.\prod_{i=1}^{\ell}\left(1 - \frac{6}{\mu^{*}} \sqrt{\frac{ \log T}{T_{i}-1}}\right)^{\frac{m_{i}}{T}} \right\rvert\, G\right]
\end{align*}

Thus, we get
\begin{equation}\label{eqn:nash_phase_two_bound}
\left(\prod_{t=\widetilde{T}+1}^{T} \mathbb{E}\left[\mu_{I_{t}}\right]\right)^{\frac{1}{T}} \geq (\mu^{*})^{1 - \frac{\widetilde{T}}{T}}\mathbb{E}\left[\left.\prod_{i=1}^{\ell}\left(1 - \frac{6}{\mu^{*}} \sqrt{\frac{ \log T}{T_{i}-1}}\right)^{\frac{m_{i}}{T}} \right\rvert\, G\right]\,\mathbb{P}\{G\}.
\end{equation}

From Remark \ref{rem:mu_min_t_tilde_bound} and the good event $G_1$ (which gives $T_i - 1 \geq \Tilde{T}/2k$), we have that
\begin{equation}\label{eqn:beta_mu_star_ratio_bound}
\frac{6}{\mu^*}\sqrt{\frac{\log T}{T_i - 1}} \leq \frac{6}{\mu^*}\sqrt{\frac{2k\log T}{\Tilde{T}}} \leq \frac{6}{64}\sqrt\frac{\Tilde{T}}{2 k \log T}\sqrt{\frac{2k\log T}{\Tilde{T}}} < \frac{1}{2}.
\end{equation}

So we can apply Claim 2 from \cite{barman2023fairness} to get 

\begin{align*}
\mathbb{E}\left[\left.\prod_{i=1}^{\ell}\left(1-\frac{6}{\mu^{*}} \sqrt{\frac{\log T}{\left(T_{i}-1\right)}}\right)^{\frac{m_{i}}{T}} \right\rvert\, G\right] & \geq \mathbb{E}\left[\left.\prod_{i=1}^{\ell}\left(1-\frac{12 m_{i}}{\mu^{*}T} \sqrt{\frac{\log T}{\left(T_{i}-1\right)}}\right) \right\rvert\, G\right]\\
& \geq \mathbb{E}\left[\left.\prod_{i=1}^{\ell}\left(1-\frac{12}{\mu^{*}T} \sqrt{\frac{m_{i} \log T}{1}}\right) \right\rvert\, G\right]
\end{align*}

Where we use the fact that $T_i - 1 \geq m_i$. Now using the inequality $(1-x)(1-y) \geq 1-x-y$ for all $x,y \geq 0$, we can write

\begin{align*}
\mathbb{E}\left[\left.\prod_{i=1}^{\ell}\left(1-\frac{12}{\mu^*T} \sqrt{\frac{m_{i} \log T}{\mu^{*}}}\right) \right\rvert\, G\right] &\geq \mathbb{E}\left[\left.1-\sum_{i=1}^{\ell}\left(\frac{12}{\mu^{*}T} \sqrt{\frac{m_{i} \log T}{1}}\right) \right\rvert\, G\right] \\ 
&=1-\left(\frac{12}{\mu^{*} T} \sqrt{\frac{\log T}{1}}\right) \mathbb{E}\left[\sum_{i=1}^{\ell} \sqrt{m_{i}} \mid G\right] \\ 
&\geq 1-\left(\frac{12}{\mu^{*} T} \sqrt{\frac{\log T}{1}}\right) \mathbb{E}\left[\sqrt{\ell} \sqrt{\sum_{i=1}^{\ell} m_{i}} \mid G\right] \\
&\hspace*{20pt}\text {(using the Cauchy-Schwarz inequality) } \\
&\geq 1-\left(\frac{12}{\mu^{*}T} \sqrt{\frac{\log T}{1}}\right) \mathbb{E}[\sqrt{\ell T} \mid G] \\
&=1-\left(\frac{12}{\mu^{*}} \sqrt{\frac{\log T}{T}}\right) \mathbb{E}[\sqrt{\ell} \mid G] \\
&\geq 1-\left(\frac{12}{\mu^{*}} \sqrt{\frac{k \log T}{T}}\right)
\end{align*}

Substituting back into equation (\ref{eqn:nash_phase_two_bound}),
\begin{align*}
\left(\prod_{t=\widetilde{T}+1}^{T} \mathbb{E}\left[\mu_{I_{t}}\right]\right)^{\frac{1}{T}} \geq\left(\mu^{*}\right)^{1-\frac{\widetilde{T}}{T}}\left(1-\frac{12}{\mu^*} \sqrt{\frac{k \log T}{T}}\right) \mathbb{P}\{G\}
\end{align*}

Finally, bounding the NSW using equations (\ref{eqn:nash_sw_decomp}) and (\ref{eqn:nash_sw_phase_one}) along with the above bound, we get 

\begin{align*}
\left(\prod_{t=1}^{T} \mathbb{E}\left[\mu_{I_{t}}\right]\right)^{\frac{1}{T}} & \geq \mu^{*}\left(1-\frac{16 \sqrt{k \log k \log T}}{\sqrt{T}}\right)\left(1-\frac{12}{\mu^{*}} \sqrt{\frac{k \log T}{T}}\right) \mathbb{P}\{G\} \\ 
& \geq \mu^{*}\left(1-\underbrace{\frac{16 \sqrt{k \log k \log T}}{\sqrt{T}}}_{\triangleq \mathsf{A}}\right)\left(1- \underbrace{\frac{16}{\mu^{*}} \sqrt{\frac{k \log T}{T}}}_{\triangleq \mathsf{B}}\right)\left(1-\frac{2}{T}\right)\tag{from Lemma \ref{lem:good_event_prob}}\\
&\hspace*{-40pt}(\mbox{since $(1 - \mathsf{A})(1 - \mathsf{B}) \geq (1 - \frac{1}{\mu^{\ast}}\mathsf{A})(1 - \sqrt{\log k}\mathsf{B}) \geq \left(1 - 2 \frac{1}{\mu^\ast}\mathsf{A}\right) = \left(1 - 2 \sqrt{\log k}\mathsf{B}\right)$})\\
&\hspace*{-40pt}(\mbox{Note that $\frac{1}{\mu^*} \mathsf{A} = \sqrt{\log k} \mathsf{B} \leq 1/2$ for sufficiently large $T$, from Assumptions \ref{assump:sufficiently_positive_rewards} and \ref{assump:sufficiently_large_exploration}})\\
& \geq \mu^{*}\left(1-\frac{32 \sqrt{k \log k \log T}}{\mu^{*}\sqrt{T}}\right)\left(1-\frac{2}{T}\right) \\
& \geq \mu^{*}-\frac{32 \sqrt{k \log k \log T}}{\sqrt{T}}-\frac{2 \mu^{*}}{T} \\ 
& \geq \mu^{*}-\frac{32 \sqrt{k \log k \log T}}{\sqrt{T}}-\frac{2}{T}
\end{align*}

Thus, from \ref{eqn:nash_regret_def_app},
\[
\textrm{NR}_T = \mu^{*} - \left(\prod_{t=1}^{T} \mathbb{E}\left[\mu_{I_{t}}\right]\right)^{\frac{1}{T}} \leq \frac{32 \sqrt{k \log k \log T}}{\sqrt{T}} + \frac{4}{T} = \Tilde{O}\left(\sqrt{\frac{k}{T}}\right).
\]
\end{proof}

%% file: appendix-p-mean-regret.tex
\section{Regret analysis of Explore-then-UCB: $p$-mean regret for $0 < p \leq 1$}\label{app:p_mean_pos}

In this section, we establish regret upper bounds for $p$-mean regret when $p$ is between $0$ and $1$. That is, we upper bound
$$(\mbox{$p$-mean Regret})\,\,\mathrm{R}^{p}_T = \mu^* - \left(\frac{\sum_{t=1}^T \left(\E_{I_t}[\mu_{I_t}]\right)^p}{T}\right)^\frac{1}{p}$$
for {$p \in (0, 1]$}. Here we give all the details in the proof, expanding the exposition in the main paper.

The following theorem establishes an ${\Tilde{O}\left(\sqrt{\frac{k}{T}}\right)}$ upper bound on $p$-mean regret when using the \textsc{Explore-Then-UCB} algorithm. 

\begin{theorem}
Given a bandit instance with $k$ arms, time horizon $T$, and regret parameter $p \in (0, 1]$, choosing the exploration period {$\Tilde{T} = 16 \sqrt{\frac{T k^p \log T}{\log k}}$}, the $p$-mean regret of the \textsc{Explore-Then-UCB} algorithm satisfies
\begin{align*}
\mathrm{R}^{p}_T &\leq \Tilde{O}\left(\sqrt{\frac{k}{T}}\right).
\end{align*}
\end{theorem}

\begin{proof}
Towards analyzing $\mathrm{R}^{p}_T$, we define 
\begin{align*}
x \triangleq \frac{\sum_{t=1}^{\Tilde{T}}\E_{I_t}[\mu_{I_t}]^p}{T} \text{ and } y \triangleq \frac{\sum_{t=\Tilde{T}+1}^{T}\E_{I_t}[\mu_{I_t}]^p}{T},
\end{align*}

so that we have
\begin{equation}\label{eqn:preg_decomp_app}
	\mathrm{R}^{p}_T = \mu^* - (x+y)^{1/p}.
\end{equation}

Hence, to obtain an upper bound for $\mathrm{R}^{p}_T$, we need to lower bound $x$ and $y$. By uniform exploration in Phase I, we have that 
\begin{align*}
\E_{I_t}[\mu_{I_t}] \geq \frac{\mu^*}{k} \Leftrightarrow \E_{I_t}[\mu_{I_t}]^p \geq \frac{{\mu^*}^p}{k^p}
\end{align*}
So we have
\begin{equation}\label{eqn:preg_x_bound}
	x \geq \frac{{\mu^*}^p \Tilde{T}}{k^p T}.
\end{equation}

We know that by Jensen's inequality ($f(z) = z^p$ for $p \in (0, 1]$ is concave on $\mathbb{R}_{>0}$) and linearity of expectation
\begin{equation}\label{eqn:preg_y_bound}
y \geq \frac{\sum_{t=\Tilde{T}+1}^{T}\E_{I_t}[\mu_{I_t}^p]}{T}  = \frac{\E_{I_1,\ldots,I_T}[\sum_{t=\Tilde{T}+1}^{T} \mu_{I_t}^p]}{T}.
\end{equation}

By reindexing the arms so that $\{1,2,\ldots,\ell\}$ are the arms pulled at least once in Phase II, and letting $m_i$ be the number of times (the reindexed) arm $i$ is pulled in Phase II, we have 

\begin{align*}
\frac{\E [\sum_{t=\Tilde{T}+1}^T \mu_{I_t}^p ]}{T} = \frac{\E [\sum_{i=1}^{\ell} m_i \mu_{i}^p]}{T}
\end{align*}

By conditioning on the good event $G$ (see Lemma \ref{lem:good_event_prob}) and noting that $\sum_{i \in [\ell]} m_i = T - \Tilde{T}$, we have
\begin{align*}
y \geq \frac{\E [\sum_{i=1}^{\ell} m_i \mu_{i}^p \rvert\, G ] \mathbb{P}\{G\}}{T}
\end{align*}

Conditioned on event $G$, we have $\mu_i \geq \mu^* - \beta_i$ for all arms $i\in [\ell]$ pulled in Phase II where $\beta_i = 6\sqrt{\frac{\log T}{T_i-1}}$. Note that $\mu_i > 2\beta_i$ from the assumptions, and $(\mu^*-\beta_i)^p \geq ({\mu^*})^p - p\beta_i({\mu^*})^{p-1}$.

Hence, we have 
\begin{align}\label{eqn:preg_y_bound_2}
y \geq \frac{ \E [\sum_{i=1}^\ell m_i ({\mu^*}^p - p\beta_i({\mu^*})^{p-1}) \rvert\, G  ] \mathbb{P}\{G\}}{T}
\end{align}

Thus, from equations (\ref{eqn:preg_x_bound}) and (\ref{eqn:preg_y_bound_2}),
\begin{align*}
    x+y &\geq \frac{{\mu^*}^p \Tilde{T}}{k^p T} + \frac{  \E [\sum_{i=1}^\ell m_i ({\mu^*}^p - p\beta_i({\mu^*})^{p-1}) \rvert\, G  ] \mathbb{P}\{G\}}{T}\\
    & = \frac{{\mu^*}^p \Tilde{T}^\prime}{\sqrt{k}T} + \frac{ \E [\sum_{i=1}^\ell m_i ({\mu^*}^p - p\beta_i({\mu^*})^{p-1})\rvert\, G  ]\mathbb{P}\{G\}}{T}\tag{taking $\Tilde{T}^\prime = \Tilde{T} / k^{p-1/2}$}\\
    & = \frac{{\mu^*}^p}{1}\left(\frac{k^{-\frac{1}{2}}\Tilde{T}^\prime}{T}\right) + \frac{{\mu^*}^p}{1} \frac{ \E \left[\sum_{i=1}^\ell m_i \left(1-\frac{p\beta_i}{\mu^*}\right) \rvert\, G  \right] \cdot \mathbb{P}\{G\}}{T}\\
    & = \frac{{\mu^*}^p}{1}\left( \frac{ k^{-\frac{1}{2}}\Tilde{T}^\prime}{T} + \frac{ \E \left[\left(1-\frac{p\beta_i}{\mu^*}\right) \sum_{i=1}^\ell m_i\rvert\, G  \right] \cdot \mathbb{P}\{G\}}{T}\right)\\
	&\triangleq (\ast_1)\\
\end{align*}
Substituting for $\beta_i$, we get
\begin{align*}
    (\ast_1) & = \frac{{\mu^*}^p}{1}\left(\frac{k^{-\frac{1}{2}}\Tilde{T}^\prime}{T} + \frac{ \E \left[\sum_{i=1}^\ell m_i - \sum_{i=1}^\ell \frac{6p m_i}{\mu^*}\sqrt{\frac{\log T}{T_i-1}}\rvert\, G  \right]\mathbb{P}\{G\}}{T}\right) \\
    & \geq \frac{{\mu^*}^p}{1}\left( \frac{k^{-\frac{1}{2}}\Tilde{T}^\prime}{T} + \frac{ \E [\sum_{i=1}^\ell m_i - \sum_{i=1}^\ell \frac{6p}{\mu^*}\sqrt{m_i \log T} \rvert\, G  ]\mathbb{P}\{G\}}{T}\right) \tag{since $T_i \geq m_i + 1$}  \\
    & = \frac{{\mu^*}^p}{1}\left( \frac{k^{-\frac{1}{2}}\Tilde{T}^\prime}{T} + \frac{ \E [T-\Tilde{T} - \frac{6p\sqrt{\log T}}{\mu^*} \sum_{i=1}^\ell \sqrt{m_i} \rvert\, G  ]\mathbb{P}\{G\}}{T}\right)\\
	&\triangleq (\ast_2)
\end{align*}
Now, using Cauchy-Schwarz inequality and $\sum_{i\in[\ell]}m_i\leq T$, we get
\begin{align*}
 (\ast_2) & \geq \frac{{\mu^*}^p}{1}\left( \frac{k^{-\frac{1}{2}}\Tilde{T}^\prime}{T} + \frac{ \E [T-\Tilde{T} - \frac{6p\sqrt{\log T}}{\mu^*} \sqrt{\ell T} \rvert\, G  ]\mathbb{P}\{G\}}{T}\right)\\
    & \geq \frac{{\mu^*}^p}{1}\left( \frac{k^{-\frac{1}{2}}\Tilde{T}^\prime}{T} + \frac{ \E [T-\Tilde{T} - \frac{6p\sqrt{\log T}}{\mu^*} \sqrt{\ell T} \rvert\, G  ]}{T}\right)\mathbb{P}\{G\}\\
    & \geq \frac{{\mu^*}^p}{1}\left( \frac{k^{-\frac{1}{2}}\Tilde{T}^\prime}{T} + \frac{T-\Tilde{T}}{{T}} - \frac{6p\sqrt{k\log T}}{\sqrt{T}\mu^*} \right)\mathbb{P}\{G\}\\
    & \geq \frac{{\mu^*}^p}{1}\left( \frac{k^{-\frac{1}{2}}\Tilde{T}^\prime}{T} + \frac{T-\Tilde{T}^\prime k^{p-\frac{1}{2}}}{{T}} - \frac{6p\sqrt{k\log T}}{\sqrt{T}\mu^*} \right)\mathbb{P}\{G\}\tag{as $\Tilde{T}^\prime \triangleq \Tilde{T}/k^{p-1/2}$}\\
    & = \frac{{\mu^*}^p}{1}\left( 1- \frac{(k^p-1)\Tilde{T}^\prime}{\sqrt{k}T} - \frac{6p\sqrt{k\log T}}{\sqrt{T}\mu^*} \right)\mathbb{P}\{G\}\\
    &\geq \frac{{\mu^*}^p}{1}\left( 1- \left(\frac{(k^p-1)\Tilde{T}^\prime}{\sqrt{k}T} + \frac{6p\sqrt{k\log T}}{\sqrt{T}\mu^*} \right)\right)\left(1-\frac{4}{T}\right)
\end{align*}

Then, 

\begin{align}\label{eqn:xplusy_pow_oneoverp_bound}
(x+y)^{\frac{1}{p}} &\geq \frac{{\mu^*}}{1}\left(1 - \left(\frac{(k^p-1)\Tilde{T}^\prime}{\sqrt{k}T} + \frac{6p\sqrt{k\log T}}{\sqrt{T}\mu^*} \right)\right)^{\frac{1}{p}}\left(1-\frac{4}{T}\right)^{\frac{1}{p}}
\end{align}

We can apply the Weierstrass inequality --- $\prod_i(1 - x_i)^{w_i} \geq (1 - \sum_i w_i x_i)$ as long as $x_i \in [0, 1]$ and $w_i \geq 1$ --- on (\ref{eqn:xplusy_pow_oneoverp_bound}) to get
\begin{align*}
(x+y)^{\frac{1}{p}} \geq {\mu^*} - \frac{(k^p-1)\Tilde{T'}}{p\sqrt{k}T} -\frac{6\sqrt{k\log T}}{\sqrt{T}} - \frac{4}{pT}
\end{align*}

Then, substituting for $\Tilde{T}$,
\begin{align*}
\textrm{R}^p_T &= \mu^* - (x+y)^{1/p}\\
&\leq \frac{(k^p-1)\Tilde{T'}}{p\sqrt{k}T} + \frac{6\sqrt{k\log T}}{\sqrt{T}} + \frac{4}{pT}\\
&= \frac{(k^p-1)\Tilde{T}}{pk^{p}T} + \frac{6\sqrt{k\log T}}{\sqrt{T}} + \frac{4}{pT}\\
&= \frac{16(k^p-1)\sqrt{T k^p \log(T)}}{pk^{p}T\sqrt{\log k}} + \frac{6\sqrt{k\log T}}{\sqrt{T}} + \frac{4}{pT}\\
&\leq \Tilde{O}\left(\sqrt{\frac{k}{T}}\right).
\end{align*}
\end{proof}

%% file: appendix-experiments.tex
\section{Numerical Experiments}\label{app:experiments}

To further validate our theoretical results about Explore-then-UCB by considering the practical effectiveness of the algorithm, we perform some basic experiments on synthetic bandit instances. The code for the experiments can be found in the repository: \url{https://github.com/philips-george/p-mean-regret-stochastic-bandits}.

\paragraph*{Instances} We consider the following synthetic stochastic bandit instances, each with $k = 50$ arms:
\begin{enumerate}
\item \textbf{Bernoulli}: The reward for each arm $i$ is distributed according to $\mathsf{Bernoulli}(\rho_i)$, so that the expected reward $\mu_i$ for arm $i$ is $\rho_i$. The means $\rho_1, \ldots, \rho_k$ are chosen independently from the uniform distribution on $[0.005, 1)$.
\item \textbf{Triangular}: The reward for each arm $i$ is drawn from a triangular distribution on $(0, 1)$ with mode $\gamma_i$. The expected reward for arm $i$ will thus be $\frac{\gamma_i + 1}{3}$. Each mode $\gamma_i$ is chosen independently from the uniform distribution on $[0.005, 0.995)$.
\item \textbf{Beta}: The reward for each arm $i$ is drawn from the Beta distribution on $(0, 1)$ with parameters $(\alpha_i, \beta_i)$. The expected reward for each arm $i$ will thus be $\frac{\alpha_i}{\alpha_i + \beta_i}$. The $\alpha_i$s and $\beta_i$s are chosen independently from the uniform distribution on $[0.005, 0.995)$.
\item \textbf{Uniform}: The reward for arm $i$ is drawn from the uniform distribution on $[l_i, u_i]$; thus $\mu_i$ will be $\frac{l_i+u_i}{2}$. The parameters $l_i$ are drawn independently from the uniform distribution on $[0.005, 0.995)$. The upper bounds $u_i$ are drawn independently of each other uniformly from the intervals $[\ell_i + 0.001, 1)$.

\paragraph{Experiments}: We evaluate three algorithms, labelled \textsf{EUCB} (which is our Explore-then-UCB algorithm), \textsf{NCB} (which is the Nash Confidence Bound algorithm from \cite{barman2023fairness}) and \textsf{UCB1} (standard UCB with one round of round-robin exploration).

We run each of these algorithms on four instances as described above (the randomness in instance generation has been fixed using a seed in the code). We run each algorithm $30$ times and take the average of the regrets. In each run, we use $T=20,000$ rounds for Triangular, Beta and Uniform instances, and $T=100,000$ for Bernoulli instances (labeled Bernoulli-100k in the results).

We report the $p$-mean regrets by running the algorithms separately for  $p = 1$ (the usual average regret), $p = 0.5$, $p = 0$ (Nash regret), $p = -0.5$, $p = -1$ (Harmonic regret), and $p = -2$.

The results of the experiments are shown in Table \ref{tab:expt_results}. We highlight in bold the results where our algorithm (\textsf{EUCB}) performs at least as well as the two other algorithms. Note that the general trend is that our algorithm outperforms or is as good as the other baselines for negative $p$ (especiallly so in the case of Bernoulli-100k instances where we use a larger number of rounds), whereas for $p = 1$ and $0.5$, the basic \textsf{UCB1} outperforms both our algorithm and \textsf{NCB}. This is to be expected since the increased uniform exploration which is part of both \textsf{EUCB} and \textsf{NCB} acts as a handicap when $p$ is sufficiently positive.

\begin{table}[t]
\caption{Experiment results.\label{tab:expt_results}}
\centering
\begin{tabular}{@{}llllll@{}}
\toprule
p                     & Algorithm & Bernoulli-100k & Triangular     & Beta           & Uniform        \\ \midrule
\multirow{3}{*}{1}    & UCB1      & 0.214          & 0.139          & 0.306          & 0.240          \\
                      & NCB       & 0.304          & 0.154          & 0.400          & 0.333          \\
                      & EUCB      & 0.320          & 0.155          & 0.399          & 0.332          \\ \midrule
\multirow{3}{*}{0.5}  & UCB1      & 0.219          & 0.139          & 0.310          & 0.242          \\
                      & NCB       & 0.229          & 0.155          & 0.402          & 0.333          \\
                      & EUCB      & 0.322          & \textbf{0.139} & 0.312          & 0.247          \\ \midrule
\multirow{3}{*}{0}    & UCB1      & 0.964          & 0.653          & 0.909          & 0.940          \\
                      & NCB       & 0.964          & 0.653          & 0.909          & 0.940          \\
                      & EUCB      & \textbf{0.964} & \textbf{0.653} & \textbf{0.909} & \textbf{0.940} \\ \midrule
\multirow{3}{*}{-0.5} & UCB1      & 0.964          & 0.141          & 0.315          & 0.245          \\
                      & NCB       & 0.359          & 0.157          & 0.407          & 0.335          \\
                      & EUCB      & \textbf{0.230} & \textbf{0.141} & 0.316          & \textbf{0.244} \\ \midrule
\multirow{3}{*}{-1}   & UCB1      & 0.964          & 0.142          & 0.320          & 0.247          \\
                      & NCB       & 0.377          & 0.142          & 0.409          & 0.244          \\
                      & EUCB      & \textbf{0.232} & 0.158          & \textbf{0.317} & 0.336          \\ \midrule
\multirow{3}{*}{-2}   & UCB1      & 0.964          & 0.144          & 0.325          & 0.255          \\
                      & NCB       & 0.409          & 0.160          & 0.415          & 0.337          \\
                      & EUCB      & \textbf{0.242} & \textbf{0.144} & \textbf{0.324} & \textbf{0.247} \\ \bottomrule
\end{tabular}
\end{table}

\end{enumerate}